\pdfoutput=1
\documentclass[twoside,11pt]{article}
\usepackage{pgm}

\usepackage[utf8]{inputenc}
\inputencoding{utf8}

\usepackage{hyperref,color,soul,booktabs}
\setulcolor{blue}
\newcommand{\BibTeX}{\textsc{B\kern-0.1emi\kern-0.017emb}\kern-0.15em\TeX}
\usepackage{algpseudocode}
\usepackage{comment}
\usepackage[linesnumbered,ruled,vlined]{algorithm2e}
\usepackage[export]{adjustbox}
\usepackage{wrapfig}
\usepackage{subcaption}
\usepackage{natbib}
\usepackage{sidecap}
\usepackage{array}
\usepackage{tabularx}
\usepackage{latexsym}
\usepackage{amsmath}
\usepackage[textwidth=6.5in, textheight=9.0in, bottom=1.0in, left=1.0in, nofoot]{geometry}
\usepackage[normalem]{ulem}
\usepackage{url}
\usepackage{enumitem}
\usepackage{microtype}
\usepackage{thmtools,thm-restate}

\newtheorem{defn}{Definition}

\newcommand{\Scal}{\mathcal{S}}

\DeclareRobustCommand{\rspmn}{\textsf{RSPMN}}

\begin{document}

\addtolength{\footskip}{20pt}

\title{Recurrent Sum-Product-Max Networks for Decision Making in Perfectly-Observed Environments}

\author{\Name{Hari Teja Tatavarti} \Email{contactme.hariteja@uga.edu}\and
  \Name{Prashant Doshi} \Email{pdoshi@uga.edu} \and
    \Name{Layton Hayes} \Email{layton.hayes25@uga.edu}\\ 
   \addr Institute for AI, University of Georgia, Athens, GA 30602}

\maketitle

\begin{abstract}
Recent investigations into sum-product-max networks (SPMN) that generalize sum-product networks (SPN) offer a data-driven alternative for decision making, which has predominantly relied on handcrafted models. SPMNs computationally represent a probabilistic decision-making problem whose solution scales linearly in the size of the network. However, SPMNs are not well suited for {\em sequential} decision making over multiple time steps. In this paper, we present recurrent SPMNs (RSPMN) that learn from and model decision-making data over time. RSPMNs utilize a template network that is unfolded as needed depending on the length of the data sequence. This is significant as RSPMNs not only inherit the benefits of SPMNs in being data driven and mostly tractable, they are also well suited for sequential problems. We establish conditions on the template network, which guarantee that the resulting SPMN is valid, and present a structure learning algorithm to learn a sound template network. We demonstrate that the RSPMNs learned on a testbed of sequential decision-making data sets generate MEUs and policies that are close to the optimal on perfectly-observed domains. They easily improve on a recent batch-constrained reinforcement learning method, which is important because RSPMNs offer a new model-based approach to offline reinforcement learning.
\end{abstract}


\vspace{-0.1in}
\section{Introduction}
\label{sec:intro}
\vspace{-0.05in}

Arithmetic circuits~\citep{huang2006solving} and sum-product networks (SPN)~\citep{poon2011sum} 
 directly learn a network polynomial that is graphically represented as a network  of sum  and product nodes from domain data.
 Evaluations  of the polynomial provide the joint or conditional  distributions as desired. These graphical models are appealing because most types of inference can be performed in  time that is  {\em linear} in  the size  of the network. On the other hand, inference in Bayesian networks is generally exponential. A limitation of SPNs is that the size of the learned network is not bounded. 
 
Given the overall benefit of these generative models, Melibari et al.~(\citeyear{Melibari16:Dynamic}) introduced recurrent SPNs as a generalization of SPNs for modeling sequence data of varying length. In particular, if a recurrent SPN is a valid SPN, inference queries can be answered in linear time, thereby providing a way to perform tractable inference on sequence data.

Sum-product-max networks (SPMN)~\citep{Melibari16:Sum} generalize SPNs by introducing two new types of nodes to an SPN: max and utility nodes. Max nodes correspond to decision variables and utility nodes to the reward function, which allow SPMNs to computationally represent a probabilistic decision-making problem. If the SPMN learned from data is valid by satisfying a set of properties, then it correctly encodes a function that computes the maximum expected utility given the partial order between the variables. As such, valid SPMNs potentially represent a shift in paradigm for decision-making models: from being primarily handcrafted to enabling machine learning from decision-making data.  

Motivated by these recent generalizations of the SPN, we present a new graphical model that extends the twin benefits of an SPN (tractable inference and directly learned from data) to a new class of problems. This new model, which we refer to as a recurrent SPMN (\rspmn{}) can be seen as a synthesis of a recurrent SPN and an SPMN: it allows extending the decision-making problem across multiple time steps thereby modeling {\em sequential} decision-making problems for the first time. Given decision-making data consisting of finite temporal sequences of values of state and utility variables, and decisions, we present an effective method for learning an \rspmn{} of any finite length from this data and evaluating it to obtain the maximum expected utility (MEU) and the corresponding policy. A key component of the learned model is the {\em template network}, whose repeated application makes the temporal generalization possible.   

We prove that unfolding the learned \rspmn{} produces a valid SPMN, which, in combination with a result from Melibari et al.~(\citeyear{Melibari16:Dynamic}) establishes that its evaluation is equivalent to using the sum-max-sum rule. On a testbed of decision-making datasets from simulations in perfectly-observed domains, we demonstrate that the learned \rspmn{}s generate MEUs that are close to the optimal. \rspmn{}s offer a model-based approach to offline (batch) reinforcement
learning where simulation data has already been collected. Consequently, we also compare the MEUs with those from a recent batch-constrained Q-learning method~\citep{fujimoto2019benchmarking} and report favorable results.        

\vspace{-0.1in}
\section{Background on SPMNs}
\label{sec:background}
\vspace{-0.05in}

We briefly review SPNs followed by its generalization to decision-making contexts, SPMNs. An SPN~\citep{poon2011sum} over $n$ random variables $X_1$, $\ldots$, $X_n$ is a rooted directed acyclic graph whose leaves are the distributions of the random variables 
and whose internal nodes are sums and products. Each edge emanating from a sum node has a non-negative weight. The value of a product node is the product of the values of its children. The value of a sum node is the weighted sum of its children's values. The value of an SPN is the value of its root. This value is the output of a network polynomial whose variables are the indicator variables and the coefficients are the weights~\citep{darwiche2000differential}. The polynomial represents the joint probability distribution over the variables if the SPN is valid. Completeness and decomposability are sufficient conditions for validity. Both impose some conditions on the scope of each node, defined below.

\vspace{-0.1in} 
\begin{defn} [Scope] The scope of a node is the union of scopes of its children, where the scope of a leaf node is itself.
\vspace{-0.1in} 
\end{defn}
In other words, the scope is the set of variables that appear in the sub-SPN rooted at that node. Next, we define the conditions that must hold for an SPN to be valid. 

\vspace{-0.1in} 
\begin{defn} [Sum-complete]
An SPN is complete iff all children of the same sum node have the same scope.
\label{def:sum-complete}
\vspace{-0.1in} 
\end{defn}

\vspace{-0.3in} 
\begin{defn} [Decomposable]
An SPN is decomposable iff no variable appears in more than one child of a product node.
\label{def:decomposable}
\vspace{-0.1in} 
\end{defn}

As it is difficult for handcrafted SPNs to meet these conditions, various structure and parameter learning algorithms have been presented to learn valid SPNs from data~\citep{poon2011sum,adellearning,gens2013learning,lowd2013learning}. Most types of inference on the structure thus learned is tractable in the size of the network.

 SPMNs~\citep{Melibari16:Sum} generalize SPNs by introducing two new types of nodes to SPNs. Max nodes that represent decision variables and utility nodes to represent the utility function. 
 An SPMN over decision variables
$D_1, \ldots, D_m$, random variables $X_1, \ldots, X_n$, and utility functions $U_1, \ldots, U_k$ is a rooted directed acyclic graph. Its leaves are either distributions over random variables or utility nodes that hold constant values. An internal node of an SPMN is either a sum, product, or max node. Each max node corresponds to one of the decision variables and each outgoing edge from a max node is labeled with one of the possible values of the corresponding decision variable. Value of a max node $i$ is $max_{j\epsilon Children(i)} v_j$, where $Children(i)$ is the set
of children of $i$, and $v_j$ is the value of the subgraph rooted at child $j$. 
    
Recall the concepts of information sets and partial ordering in influence diagrams~\citep{koller2009probabilistic}. Information sets $I_0$, $\ldots$, $I_m$ are subsets of the random variables such that the random variables in the information set $I_{i-1}$ are observed before the decision associated with variable $D_i$, $1 \leq i \leq m$, is made. Any information set may be empty and variables in $I_m$ need not be observed before some decision node. An ordering between the information sets may be established as follows:
$I_0$ $\prec$ $D_1$ $\prec$ $I_1$ $\prec$ $D_2$, $\ldots$, $\prec$ $D_m$ $\prec$ $I_m$. This is a partial order, denoted by $P^\prec$, because variables within each information set may be observed in any order. Melibari et al. show that a set of properties are needed to ensure that an
SPMN correctly encodes a function that computes the MEU given the partial order between the variables and some utility function $U$. In particular, an SPMN is valid if it satisfies Defs.~\ref{def:sum-complete} and~\ref{def:decomposable}, and two new additional properties: 

\vspace{-0.1in} 
\begin{defn}[Max-complete]
An SPMN is max-complete iff all children of the same max node have the same scope, where the scope is as defined previously.
\label{def:max-complete}
\vspace{-0.1in} 
\end{defn}
\vspace{-0.1in} 
\begin{defn}[Max-unique] An SPMN is max-unique iff each max node that corresponds to a decision variable D appears at most once in every path from root to leaves.
\label{def:max-unique}
\vspace{-0.1in} 
\end{defn}

An SPMN is solved by assigning values to the random variables that are consistent with the evidence. Then, we perform a bottom-up pass of the network during which operators at each node are applied to the values of the children. The optimal decision rule is found by tracing back (i.e., top-down) through the network and choosing the edges that maximize the decision nodes.

\vspace{-0.1in}
\section{Recurrent SPMNs}
\label{sec:RSPMNs}
\vspace{-0.1in}

Popular frameworks such as a Markov decision process (MDP) and languages such as dynamic influence diagrams~\citep{shachter2010dynamic} model long-term decision making as a temporal sequence of decision-making steps. For our purposes, each of these steps can be modeled using a structure analogous to an SPMN. This could yield a structure that is similar to a dynamic influence diagram, which unfolds an influence diagram with temporal links as many times as the number of steps in the extended problem thereby generating a much larger influence diagram that models the complete sequence. 
  
We take this perspective to modeling sequential decision making and introduce a {\em recurrent SPMN} (\rspmn{}), which unfolds a {\em template network} as many times as the number of time steps in each sequence of data. While the template network is not rooted at a single node and is not a valid SPMN, we obtain these by learning an  additional component: a top network that caps the unfolded templates, which, in conjunction with some properties on the structure of the template then yields a valid SPMN.

An alternative approach to the recurrent SPMN is to directly learn the SPMN from the sequence data using the LearnSPMN algorithm~\citep{Melibari16:Sum}. However, this poses two main challenges. First, an increase in the sequence length often leads to an exponential blow up of the size of the network and subsequently in evaluation time as we demonstrate later in our experiments. Second, the LearnSPMN algorithm requires a fixed number of variables in each data record. Hence, it may not be used when the sequence length varies between records as there may not always be an efficient way to either fill in the missing time steps for shorter sequences or eliminate extra sequences from the longer ones.

We begin by describing which domain attributes should be present in the data to allow learning RSPMNs followed by formal definitions and illustrations of the components of the RSPMN. 

\subsection{Data Schema}
\label{subsec:schema}

Useful data for learning RSPMNs consists of a finite temporal sequence of values of state and utility variables, and decisions that are actions. More formally, consider a decision-making problem where the (fully observed) state of the environment is characterized by $n$ variables, $X_1$, $X_2$, $\ldots$, $X_n$; decisions by a combination of $m$ decision variables,
$D_1$, $D_2$, $\ldots$, $D_m$; and a single utility variable $U$. A candidate data record of at most $T$ steps is then a sequence of $T$ tuples of the form $\langle (I_0, d_1, I_1, d_2, \ldots, I_{m-1}, d_m, I_m,u)^0$, $(I_0, d_1, I_1, d_2, \ldots, I_{m-1}, d_m, I_m, u)^1$, $\ldots$,
$(I_0, d_1, I_1, d_2, \ldots, I_{m-1}, d_m, I_m, u)^{T-1})\rangle$. Recall from Section~\ref{sec:background}, $I_0, I_1, \ldots, I_m$ are information sets where $I_{i-1}$, $1 \leq i \leq m$ consists of values of the state variables in the information set of $D_i$. Additionally, $u$ in each tuple is the value of utility variable $U$ given the realizations of the state variables and decisions in that tuple.

\subsection{RSPMN Properties and Validity}

An RSPMN models sequences of decision-making data of varying lengths using a fixed set of parameters by unfolding a template network. In the context of a dynamic influence diagram, our template corresponds to an influence diagram with temporal links between nodes that are repeated in each time slice. 

\vspace{-0.1in}
\begin{defn}[Template network]
A template network 
is a directed acyclic graph with $r$ root nodes and at least $n+1$ leaf nodes where $n$ is the number of state variables and there is one utility function. The root nodes form a set of interface nodes $Ir$. The leaf nodes in the network hold the distributions over the random state variables $X_1, X_2, \dots, X_n$, hold constant values as utility nodes, or are latent interface nodes. The root interface nodes and interior nodes can either be sum, product, or max nodes. Let $L$ denote the set of leaf latent interface nodes. Each latent node in $L$ is related to a root interface node in $Ir$ of the template network through a bijective mapping   $f: L \rightarrow Ir$.
\vspace{-0.1in}
\end{defn}

 The bijective mappings can be seen as time delay edges that link and let replace latent interface nodes at time step $t$ with root interface nodes at $t+1$, thereby enabling recurrence of the template. The scope of any leaf latent node is itself. But, the scope changes when the template network is unfolded. The scope of a latent node of a template network in time step $t$ is related to the scope of a root interface node of the template network at time step $t+1$. More formally, for any pair of latent nodes $l_i^t, l_j^t \in L$, let $f(l_i^t) = ir_i^{t+1}, f(l_j^t) = ir_j^{t+1}$, where $ir_i^{t+1}, ir_{j}^{t+1} \in Ir$, then 
$ \left ( scope(ir_i^{t+1}) = scope(ir_j^{t+1}) \right ) \Rightarrow \left ( scope(l_i^t) = scope(l_j^t) \right ).$

Intuitively, the leaf latent interface nodes can be viewed as summarizing the latent information coming from the subsequent template network. They pass information between templates of different steps. In other words, they pass up the information in a bottom-up evaluation of the \rspmn{} and pass down the information in a top-down pass. As such, the root and leaf latent nodes play a key role in linking the template networks during unfolding. 

Toward ensuring that the unfolded network is a valid SPMN with a single root, we define another special network as given below.

\vspace{-0.1in}
\begin{defn}[Top network]
A top network is a rooted directed acyclic graph consisting of sum and product nodes, and whose leaves are the latent interface nodes. Edges from a sum node are weighted as in a SPN
\label{def:top_network}
\vspace{-0.1in}
\end{defn}

Of course, the bottom-most template network -- corresponding to the final time step $T$ of the sequential decision making -- has its leaf interface nodes removed. Parents of these interface nodes that are sum or product nodes with no other children are also pruned. We may effectively achieve this by setting the values of all these interface nodes as 1 (thus summing them out) and any utility values to pass set to 0.

Next, we seek to ensure that the SPMN formed after interfacing the top network and repeated templates is valid. One way to check for validity is to ensure that all the sum nodes in the unfolded SPMN are complete, the product nodes are decomposable, and max nodes are complete and unique as in defined in Section~\ref{sec:background}. However, can we define constraints on the top and template networks that will ensure validity of the unfolded SPMN? If so, we may establish the validity without checking the full network, which may grow to be quite large. To establish this, we first introduce a {\em soundness} property for the template network. 


\vspace{-0.1in}
\begin{defn}[Soundness of the template] A template network is sound iff 
all sum nodes in the template are sum-complete as defined in Def.~\ref{def:sum-complete};
all product nodes in the template are decomposable as defined in Def.~\ref{def:decomposable};
all max nodes in the template are max-unique and max-complete as defined in Defs.~\ref{def:max-complete} and~\ref{def:max-unique};
the scope of all the root interface nodes in $Ir$ is the same, i.e.,
    $scope(ir_i) = scope(ir_j) ~~~ \forall ~ir_i, ir_j \in Ir;$
and, the scopes of the leaf latent interface nodes in $L$ are related to that of the mapped root interface nodes in $Ir$.

\label{def:template-sound}
\vspace{-0.1in}
\end{defn}

Next, Theorem~\ref{thm:RSPMN-valid}  establishes that a sound template network combined with a valid top network generates a valid SPMN on unfolding the \rspmn{}. We provide the proof in the online Appendix at \url{https://github.com/c0derzer0/RSPMN}. 

\vspace{-0.1in}
\begin{restatable}[Validity of \rspmn{}]{thm}{Validity}
If, ($a$) in the top network, all sum nodes are complete and product nodes are decomposable, i.e., top network is valid, and ($b$) the template network is sound as defined in Def.~\ref{def:template-sound}, then the SPMN formed by interfacing the top network and the template network unfolded an arbitrary number of times as needed is valid.
\label{thm:RSPMN-valid}
\vspace{-0.1in}
\end{restatable}

\vspace{-0.1in}
\section{Learning of \rspmn{}s}
\label{sec:LearnRSPMN}
\vspace{-0.05in}


We present a complete algorithm for learning the structures and parameters of a valid top network and a sound template network from data whose schema was outlined in Section~\ref{subsec:schema}. Each data record of length $T$ is a capture of an episode during which a decision-maker interacts with the environment for $T$ steps (observing the state, acting, and obtaining reward). Let there be $E$ such episodes. For convenience, we denote a tuple 
$(I_0, d_1, I_1, d_2, \ldots, I_{m-1}, d_m, I_m,u)^t$ as $\tau^t$. Consequently, the data set has E records each consisting of $T$ tuples $\langle \tau^0, \tau^1, \ldots, \tau^T \rangle_e$ where $e = 1,2, \ldots, E$. Let us also note that all variables related to time step $t+1$ assume their values after time step $t$. This is reflected in the expanded $P^\prec$, which now specifies the partial order not only among variables of a single time step but also includes variables of the next time step. This is sufficient because the partial order among variables of two consecutive steps does not change over time.

\begin{algorithm}[!ht]
\small
\DontPrintSemicolon

    \SetKwInOut{Input}{Input}
    \Input{
        Dataset $\langle \tau^0, \tau^1, \ldots, \tau^T \rangle_e$ where $e = 1,2, \ldots, E$; 
        Partial Order $P^\prec$
    }
  \SetKwInOut{Output}{Output}
  \Output{top network, template network}
  
  $\Scal{}^{t=2}$ $\leftarrow$ Run {\sc LearnSPMN} over wrapped 2-time step data $\langle \tau^0, \tau^1 \rangle_{e'}$~~ $e' = 1,2, \ldots, T \times E$ \;
  Create top network and set of root interface nodes $Ir$ from $\Scal{}^{t=2}$ \;
  Create initial template network from $Ir$ and $\Scal{}^{t=2}$\;
  Revise initial template using sequence data to obtain the final template
   
\caption{{\sc LearnRSPMN}}
\label{alg:LearnRSPMN}
\vspace{-0.05in}
\end{algorithm}

Algorithm~\ref{alg:LearnRSPMN} presents the four main steps involved in learning the \rspmn{} from data. We refer to this algorithm as {\sc LearnRSPMN}. We describe each of these steps below in more detail and show the algorithms in the online Appendix. We also illustrate their applications on a simple 2$\times$2 grid problem. 

Each cell in the grid is represented by two binary state variables $X,Y$, which represent the $x, y$ coordinates of the cell, respectively. Here, the top-left cell is $(0,0)$ and $(1,1)$ indexes the bottom-right cell. The agent can decide to either move in one of the four cardinal directions or perform a No-op, which is represented using a single decision variable, $A$. Let $(0,0)$ be the start state, $(1,0)$ a penalizing state with a reward of -10, and $(1,1)$ the goal state with a reward of 10. All transitions are deterministic and cost -1. Reward is represented by the utility variable $U$. We simulated a randomly-acting agent in this domain for $T=4$ to generate a data set of 10K records with schema as given in Section~\ref{subsec:schema}.  


\vspace{-0.1in}
\paragraph{[Step 1] Learn an SPMN from 2-time step data} Sequential decision-making environments can often be modeled as Markovian. Thus, state transition probabilities and utility functions can be sufficiently learned from data spanning two time steps. Consequently, the first step of {\sc LearnRSPMN} is to use Melibari et al.'s LearnSPMN  algorithm~(\citeyear{Melibari16:Sum})

to learn a valid SPMN, $\Scal{}^{t=2}$, from 2-time step data. Subsequently, $\Scal{}^{t=2}$ serves as a basis for obtaining the template network.

However, which two time steps of the data record should we utilize? One might think that it may be sufficient to limit to tuples of the first two steps in each data record $\langle \tau^0, \tau^1 \rangle$, or to tuples of any particular two consecutive steps $\langle \tau^{t'}, \tau^{t'+1} \rangle$. But, an agent often starts the episode at the same start state and is often at the same intermediate state in a subsequent time step. As such, data in the first two time steps, or for that matter, any fixed pair of time steps, is seldom fully representative of the transition probabilities. Consequently, we consider each consecutive pair of tuples $\langle \tau^t, \tau^{t+1} \rangle$~~$t=0,1,\ldots,T-2$ in each data record and wrap it to create a data set with $T \times E$ rows spanning two time steps. 

\begin{figure}[!ht]
\centerline{\includegraphics[scale=0.75]{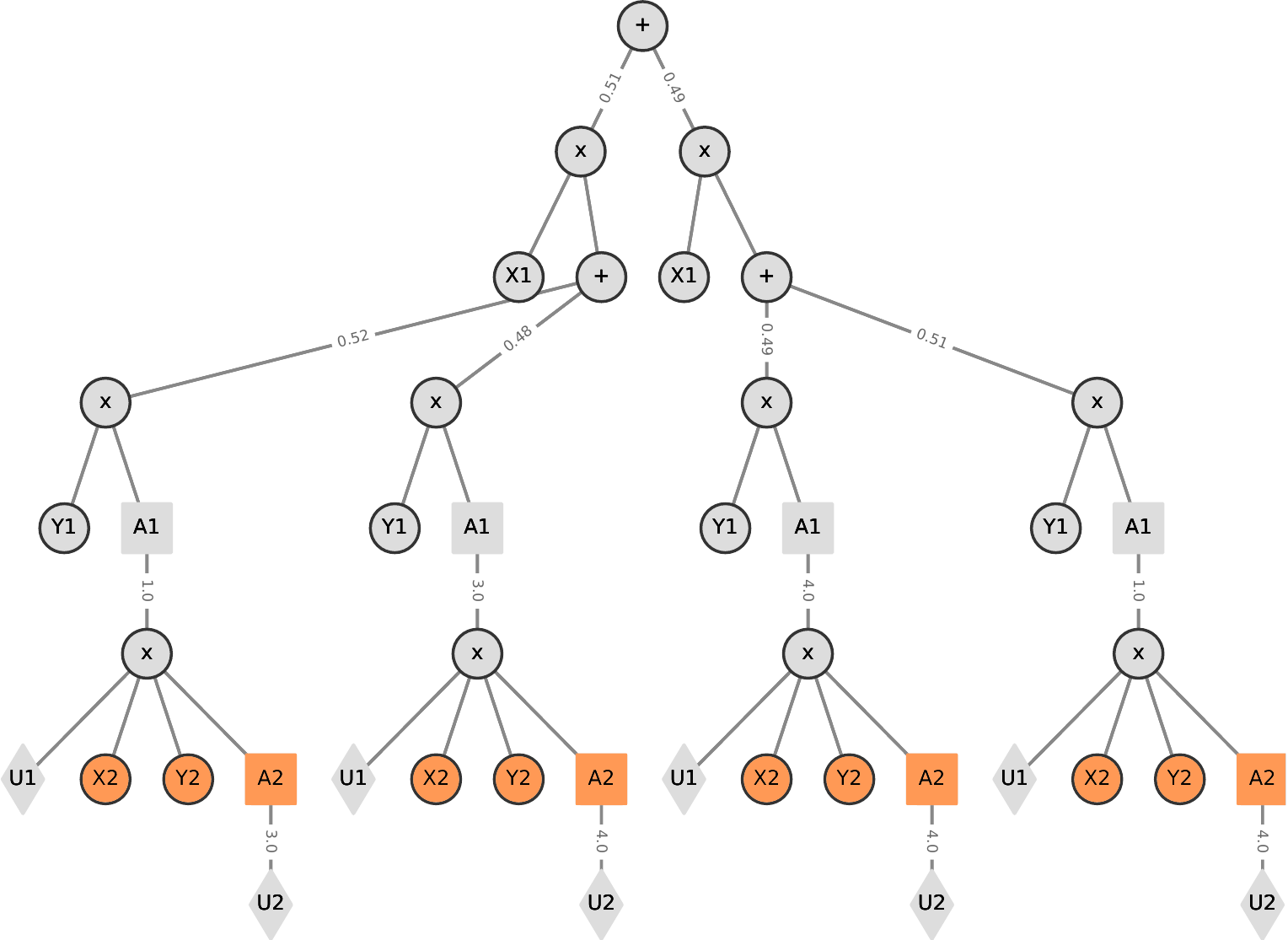}}
\caption{\small SPMN $\Scal{}^{t=2}$ learned on wrapped 2-time step data for the example grid problem. $Xt$, $Yt$, $At$, and $Ut$ represent the corresponding variables for step $t$, where $t \in \{1,2\}$.}
\label{fig:SPMN2t}
\vspace{-0.1in}
\end{figure}

Note that $P^\prec$ is focused on the partial order among the variables of two consecutive time steps, and need not change for use in  LearnSPMN. Then, LearnSPMN is run over the wrapped data set with $P^\prec$ as the partial order. We show the learned 2-time step SPMN for the example grid problem in Fig.~\ref{fig:SPMN2t}.

\paragraph{[Step 2] Obtain top network and $Ir$ nodes} To obtain the nodes in $Ir$, we extract a 1-time step network $\Scal{}^{t=1}$ from $\Scal{}^{t=2}$. We point out that it is necessary to use a 2-time step SPMN to obtain $\Scal{}^{t=1}$. To realize this, let a state-action pair 
$\langle s_0, a_0 \rangle$ transition to state $s_1$ while $\langle s_2, a_0 \rangle$ transition to $s_3$. If $\Scal{}^{t=1}$ is learned from data of a single time step, the correlations between variables of different steps is obviously not ascertained. Due to this, both state values $s_0$ and $s_2$ may get included in a single substructure. However, the 2-time step data makes it clear that they effect differing transitions, which could  be identified during independence testing, thereby modeling them separately by creating different clusters for each of these states in the first step as they result in a transition to the next step. Importantly, this helps identify the behaviorally distinct states of the domain and helps create an interface node for each of them in the template network.

\begin{figure}[!ht]
\centerline{\includegraphics[scale=0.75]{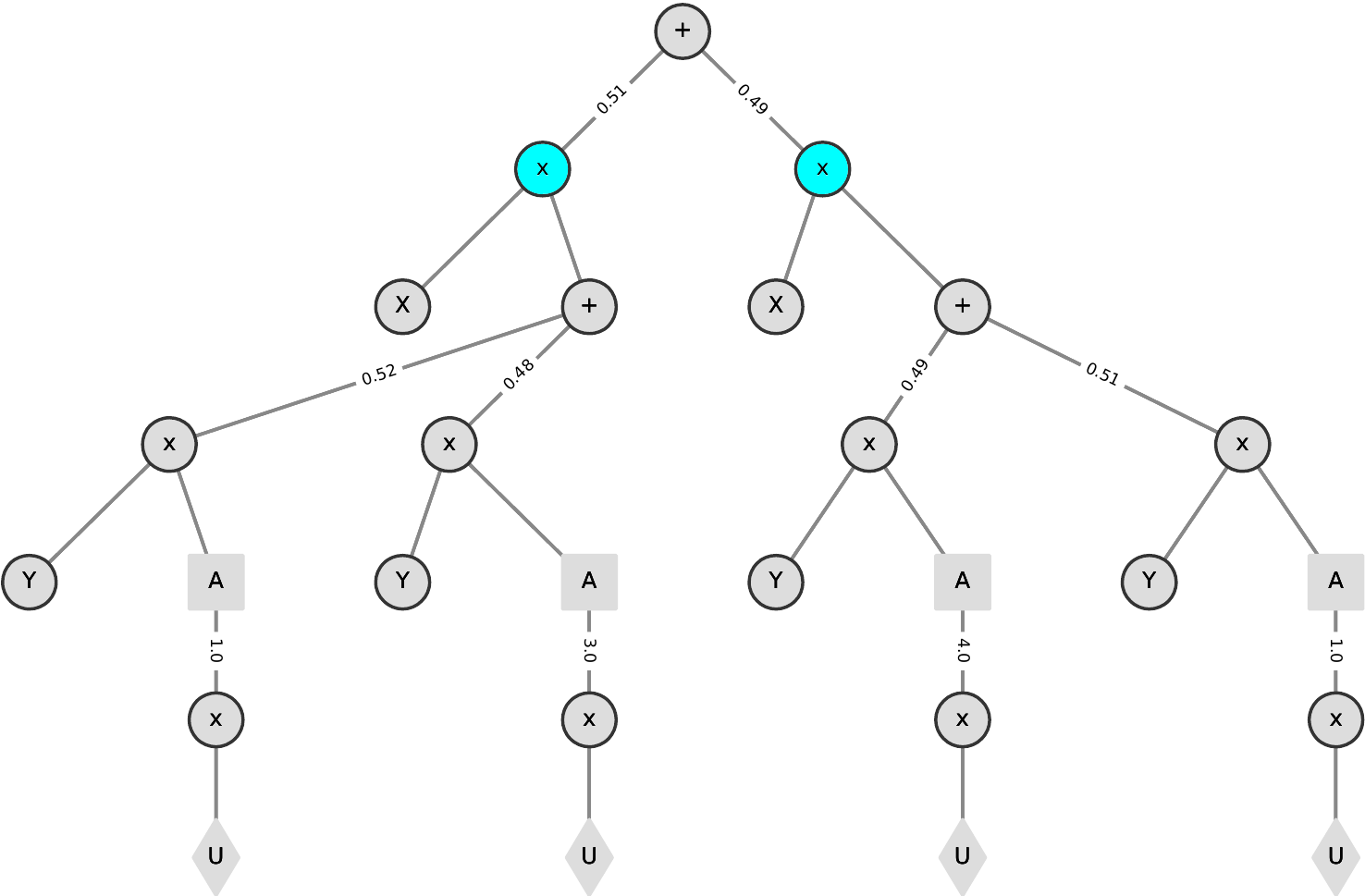}}
\caption{\small $\Scal{}^{t=1}$ obtained from the $\Scal{}^{t=2}$ of Fig.~\ref{fig:SPMN2t} with the orange nodes removed.}
\label{fig:OneStepNetwork}
\vspace{-0.1in}
\end{figure}

To obtain $\Scal{}^{t=1}$, we simply remove all those sum nodes whose immediate children have scopes that consist of subset of variables in the next time step $(I_0, d_1, I_1, d_2, \ldots, I_{m-1}, d_m, I_m, u)^1$. If there are no such sum nodes, but instead variables of the next time step are directly linked to product nodes (as in Fig.~\ref{fig:SPMN2t} where the orange nodes are children of the product nodes), then we remove the nodes corresponding to these children. This results in $\Scal{}^{t=1}$ with no nodes whose scopes lie in variables of the next time step, yet appropriately modeling its impact on the first time step. Figure~\ref{fig:OneStepNetwork} illustrates $\Scal{}^{t=1}$ for the grid problem.


Now, we may obtain the nodes in $Ir$ using $\Scal{}^{t=1}$. Starting from the top-most product nodes (the blue nodes in Fig.~\ref{fig:OneStepNetwork}), 
the root interface nodes are obtained by identifying all the distinct state distributions from these product nodes. This is done by recursively traversing all the branches of the product node until we find a product node without any sum node as a child. 
This corresponds to the four product nodes in Fig.~\ref{fig:OneStepNetwork}
below the top two blue colored nodes.
Each of these product nodes as well as leaf nodes of all the parent product nodes on path to this product node are added to a corresponding set. A product node is created for {\em each} of these sets and the elements of the set are added as children of the product node as shown in Fig.~\ref{fig:topN-InterfaceRootNodes}(a). Each of these product nodes is a root interface node. Each of these interface nodes holds an SPMN that corresponds to a state distribution (there are four interface nodes for the four states in illustration). 

The interface nodes, for example, can help learn the probability of transitioning from one state $s^t$ on taking $a^t$ to some other state $s^{t+1}$ in the next time step. {\em Observe that union of the scopes of the children of each interface node in $Ir$ is identical. This makes the scopes of all the interface nodes $Ir$ identical.}

\begin{figure}[!ht]
\begin{subfigure}{0.6\textwidth}
\centerline{\includegraphics[scale=0.6]{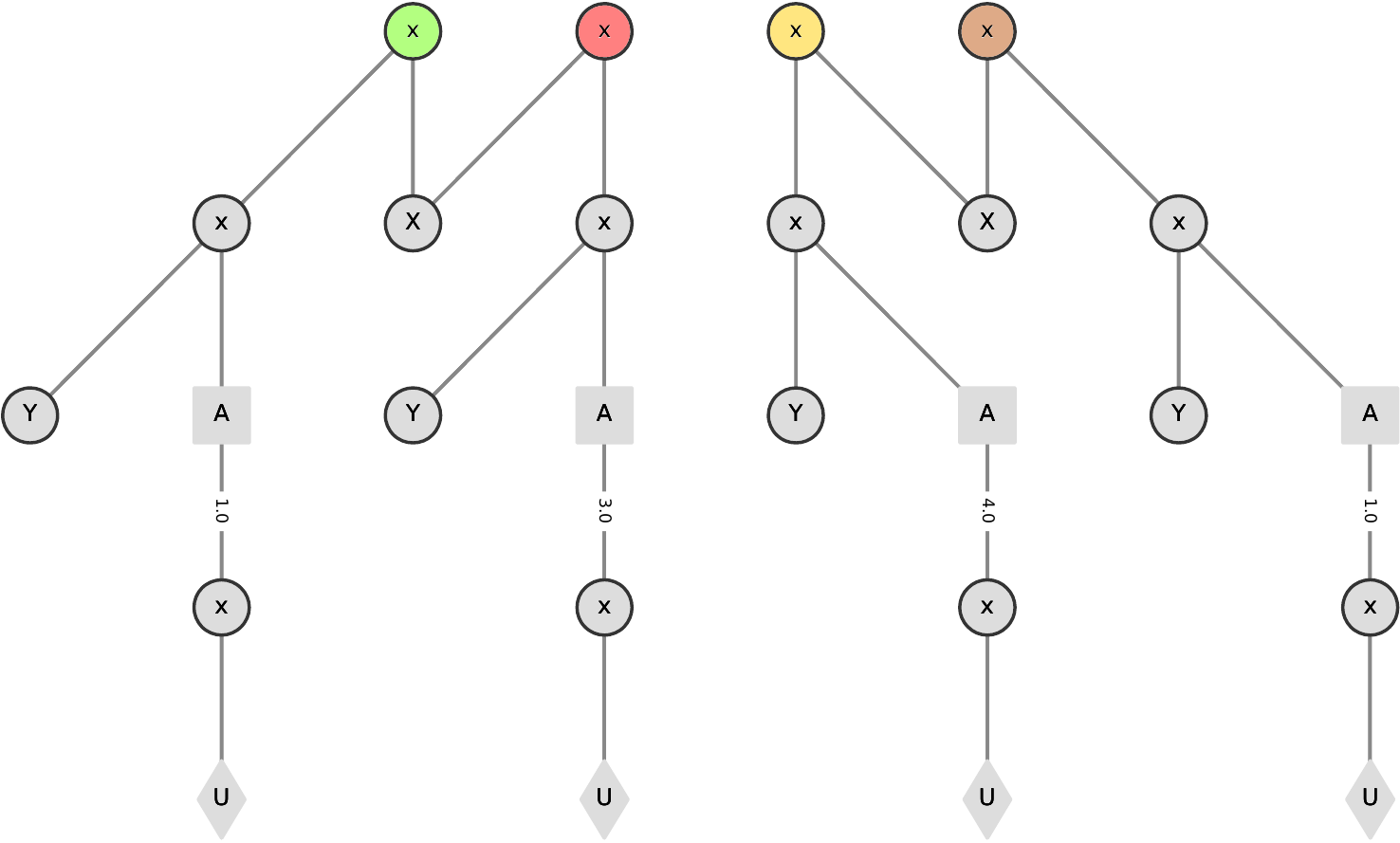}}
\caption{\small The colored product nodes function as the root interface nodes.}
\label{fig:InterfaceRootNodes}
\end{subfigure}
\begin{subfigure}{0.4\textwidth}
\centerline{\includegraphics[scale=0.5]{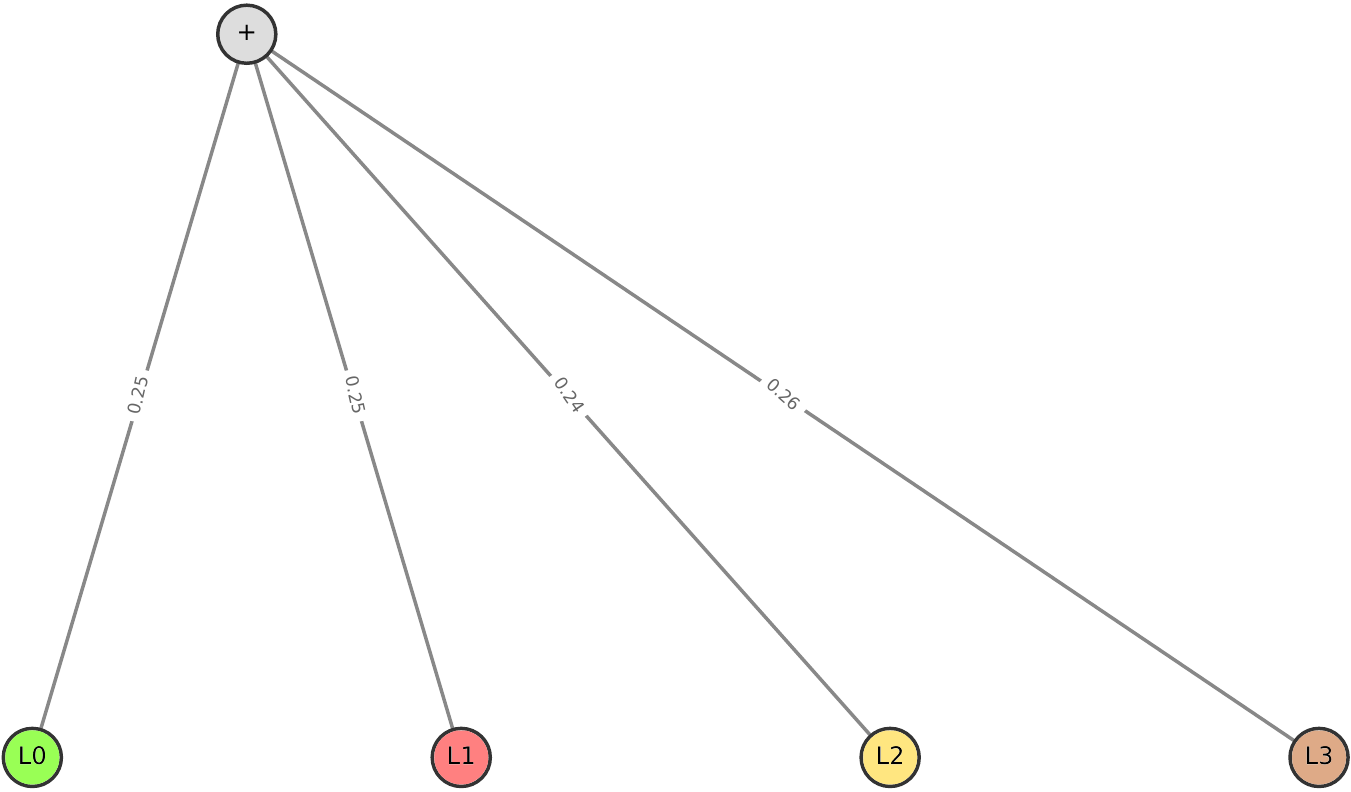}} 
\caption{\small The top network with color-coded latent interface nodes indicating a bijective relationship with the similarly color-coded nodes in $Ir$.}
\label{fig:UnprunedTopNet}
\end{subfigure}
\caption{\small Illustrations of the root interface nodes obtained using the network in Fig.~\ref{fig:OneStepNetwork} and the top network for the example grid problem.}
\label{fig:topN-InterfaceRootNodes}
\vspace{-0.1in}
\end{figure}

The top network is then simply a sum node with as many children as the nodes in $Ir$. The weights on these edges are equal and correspond to a uniform distribution. Each of these children is a latent interface node with a bijective relationship to a root interface node. We show the top network in Fig.~\ref{fig:topN-InterfaceRootNodes}(b).

\begin{figure}[ht]
\centerline{\includegraphics[scale=0.75]{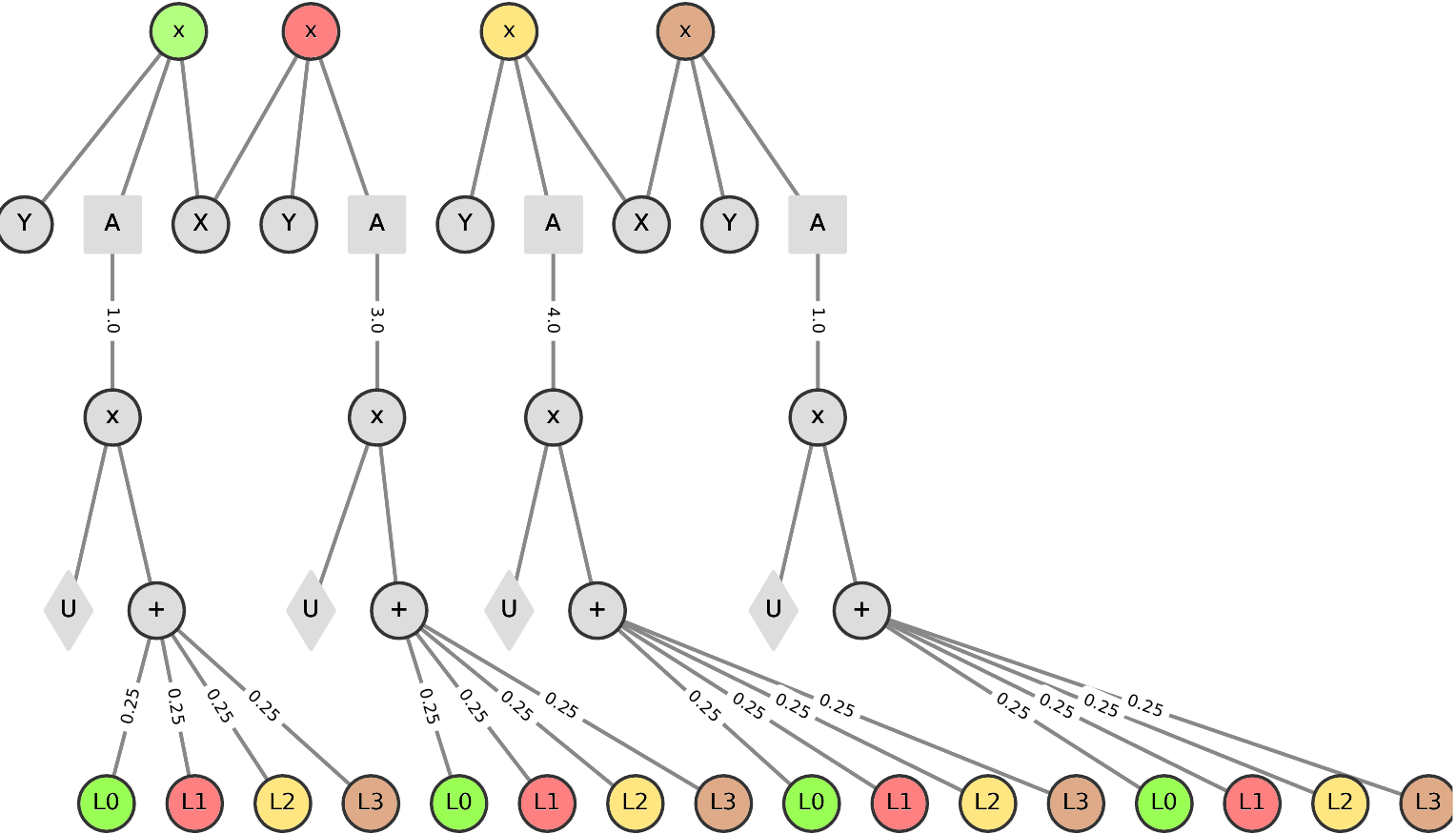}}
\caption{\small An initial template network for the grid problem obtained by attaching $S_L$ to the bottom-most product node following each node in $Ir$. Similarly colored nodes in the sets $L$ and $Ir$ are related through the bijective mapping.}
\label{fig:initial_template}
\vspace{-0.1in}
\end{figure}

\vspace{-0.1in}
\paragraph{[Step 3] Building an initial template} Let $|Ir|$ denote the number of root interface nodes created in the previous step. We begin by creating a subnetwork $S_L$ rooted at a sum node with as many children as $|Ir|$. Each of the children is a leaf {\em latent interface node} observing the following relationship, $f(l_i) = ir_i$, $i = 1, 2, \ldots, |Ir|$, and $f$ is a bijective relationship. As such, each of the latent interface node corresponds to a distinct root interface node. The weights on the edges are equal and correspond to a uniform distribution.

The network from the previous step containing the root interface nodes forms a basis for creating the initial template. Beginning at {\em each} root node in $Ir$, we traverse the graph to the bottom-most sum, product, or max node. In case of a product node, we add a new edge and link it to a new subnetwork $S_L$. In case of a sum or max node,
each of its children nodes is now replaced by a product node with two outgoing edges. One of these edges links to the previous child node while the other edge links to $S_L$. Including all latent nodes in $S_L$ can be effectively thought of as having observed a state and taken a decision at time step $t$ results in reaching all the other states in the next time step $t+1$ with equal probability. We show the initial template network for the example grid problem in Fig.~\ref{fig:initial_template}.

As each latent interface node is related to a root interface node through the bijective mapping and the root interface nodes have identical scopes, the sum node of each $S_L$ is complete. As $S_L$ is added beside every leaf node, it does not impact the properties of other nodes. Therefore, the initial template network is sound.

\vspace{-0.1in}
\paragraph{[Step 4] Learning the final template network} Parameters of the template network are the edge weights of the outgoing edges from the sum nodes including $S_L$. We adapt the hard expectation-maximization for SPNs~\citep{poon2011sum, peharz2015foundations} to the recurrent structure of the template to update the structure and parameters of the initial template. Broadly, it involves performing a bottom-up pass during which the likelihoods of each sum, product, and max node are calculated using a data record $\langle \tau^0, \tau^1, \ldots, \tau^{T-1} \rangle$.

This is followed by a top-down (backpropagation) pass beginning at the rooted top network, which selects a maximum likelihood path and updates the counts on the edges from sum nodes along that path.    

Data from the last tuple $\tau^{T-1}$ is entered in the leaf random variable nodes of the bottom-most template (recall that the bottom-most template network has its leaf latent nodes removed). Likelihoods are propagated upwards through the network by performing the sum, product, and max operations represented by the nodes until we obtain a likelihood for each in node in $Ir^{T-1}$. Using the bijection function that relates the nodes in $L$ with the nodes in $Ir$, we may propagate the likelihoods of the nodes in $Ir^{T-1}$ to the corresponding latent nodes in $L^{T-2}$. The above mentioned bottom-up pass is repeated using the likelihoods of the leaf latent nodes and data from tuple $\tau^{T-2}$ entered into the leaf random variable nodes of time step $T-2$ template, thereby yielding another set of likelihoods for the nodes. Regressing in data to time step 0, the bottom-up pass continues assigning a likelihood to each node in the template network terminating when the root node of the top network is reached. 

\begin{figure}[ht!]
\centerline{\includegraphics[scale=0.75]{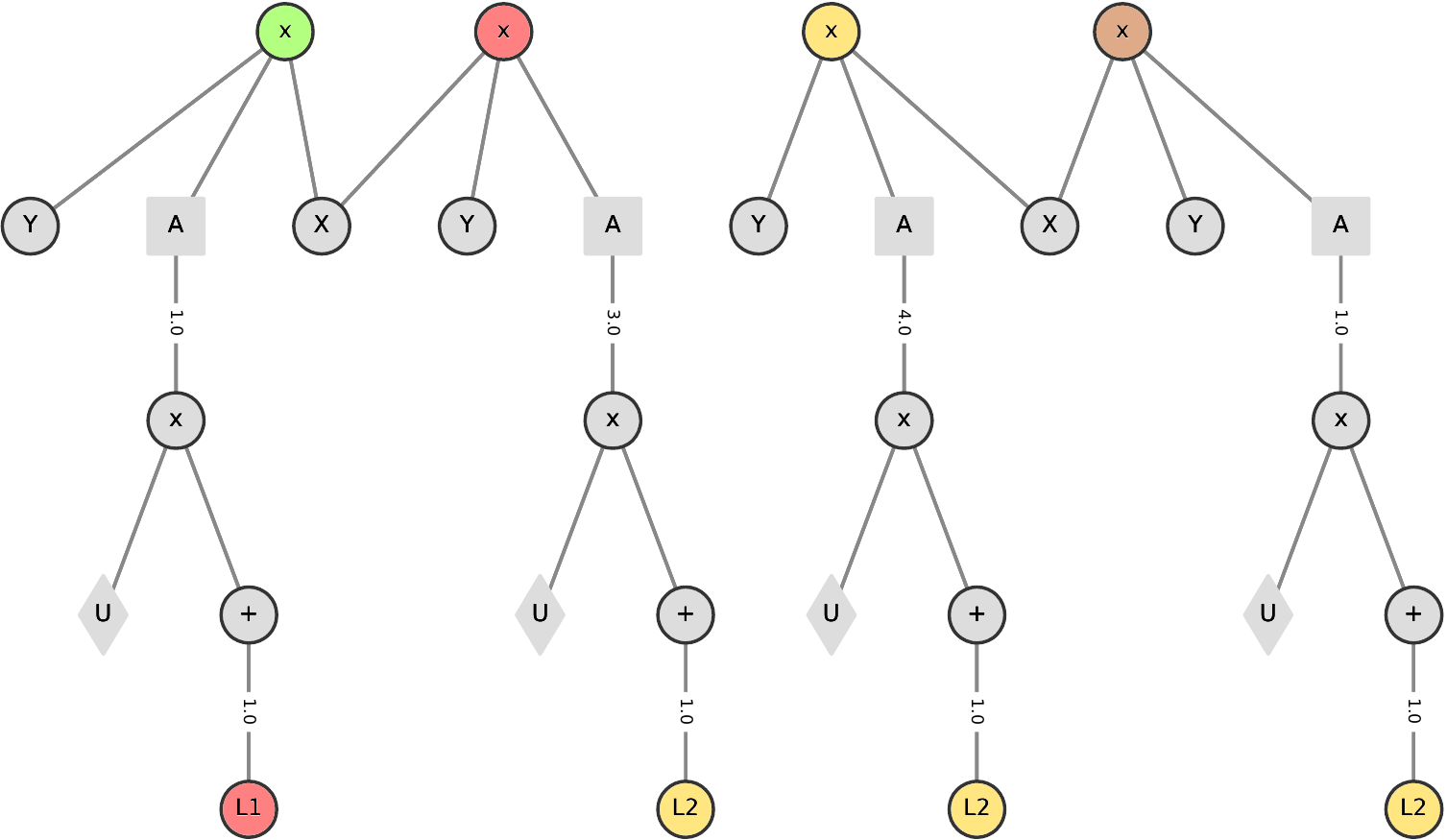}}
\caption{\small The final template network from the bottom and top-down passes for the grid problem. Notice that we retain the leaf latent nodes at each $S_L$ in the initial template with the maximum likelihoods only.}
\label{fig:FinalTemplateNetwork}
\vspace{-0.1in}
\end{figure}

Given the likelihoods computed at each node for each time step, the top-down pass begins at the root node of the top network and at time step 0. It traverses downward visiting each node, selecting the child node with the highest likelihood for each sum node (including subnetwork $S_L$) and updating the count (initialized at zero) on the edge connecting the sum node to the child, selecting the child with the highest likelihood for the max node, and following each edge of the product node. The bijection mapping is used to go from the leaf latent nodes with maximum likelihood in time step $t-1$ to the mapped root interface nodes of time step $t$. An edge from a sum node chosen again has its previous count incremented by 1. The top-down pass terminates at the bottom-most network representing time step $T-1$. 

New weights of outgoing edges from each sum node are obtained as: 
$\frac{\text{count on edge from sum $\rightarrow$ child node}}{\text{\# sum node visited}}$. {\em Thereafter, any leaf latent nodes (and indeed any children of a sum node) with zero counts are pruned.} We may perform both the bottom-up and top-down passes without actually unfolding the template network by following the implicit links represented by the bijection mapping. Applying this step, the learned final template for our illustrative grid problem is shown in Fig.~\ref{fig:FinalTemplateNetwork}. As we prune just the latent interface nodes, the scope of $S_L$ does not change and the template network remains sound.


\vspace{-0.1in}
\section{MEU and Policy Evaluation}
\label{sec:MEU}
\vspace{-0.05in}

We may evaluate the RSPMN formed by interfacing the top network with the learned final template network iterated as many times as the length of each data record to compute the MEU for each state and obtain a policy from the MEU values. 

The MEU value is obtained by evaluating the template network bottom-up as in an SPMN. The utility values of the leaf latent interface nodes of the bottom-most template (last time step) are set to zero. After evaluating the bottom-most network, each root interface node of the template network holds a utility value. In the next iteration, the utility values of the leaf latent interface nodes are assigned the utility values of the corresponding root nodes computed in the previous iteration, and the process is repeated until time step 0 and the top network is evaluated. Assuming that the template network learned a model of the true transitions well, each bottom-up pass through the template can be thought of as performing one Bellman update in the value iteration technique. This can be run until the desired length of the sequential problem is reached. 
Subsequently, each node of the template holds a corresponding expected utility value. 

To get the best action given an observation, the template network is interfaced with the top network and variables are assigned values corresponding to the observed state. A top-down pass starting at the root node of the top network and choosing the action with the MEU at each max node reached  as in an SPMN, the best action(s) comprised of the decision at each max node is obtained for that state. 

\vspace{-0.1in}
\section{Experiments}
\label{sec:experiments}
\vspace{-0.05in}

We implemented {\sc LearnRSPMN} in the SPFlow library~\citep{Molina2019SPFlow} and evaluated its performance on a new testbed of sequential decision-making data sets that adhere to the schema given in Section~\ref{subsec:schema}.

\vspace{-0.1in}
\paragraph{Evaluation testbed} There are few existing data sets on simulations of decision-making domains. Due to this, we created a new testbed of seven data sets, listed in Table~\ref{tbl:Datasets_networksizes} and available for download at \url{https://github.com/c0derzer0/RSPMN}. Four of these data sets are simulations of these domains present in OpenAI's Gym and the remaining three are simulations of RDDLSim~\citep{Sanner:RDDL} MDP domains. Each data set is generated by using a random policy which interacts with the environment and collecting the $\langle$state, action, reward$\rangle$ generated at each step. Each episode is run for $T$ time steps, which is selected to be sufficient to reach the goal state. A new episode is started if either the last time step is reached or if the agent reaches the goal state or some other terminal state.

\begin{table}[!ht]
\begin{small}
\begin{center}
\begin{tabular}{l|c|c|c|c|c|c}
\bf Data set & \bf $|X|$, $|D|$ & \bf \#Episodes & \bf $T$ & \bf $|$Columns$|$ & \bf $|$SPMN$|$ & \bf $|$RSPMN$|$\\ 
\hline
GridUniverse$^1$   & (1, 1)   & 100K  & 8     & 24    & 138,492   & (13, 210)
\\
FrozenLake$^1$     & (1, 1)     & 100K  & 8     & 24    & 1,068,246 & (18, 401) \\
Maze$^1$           & (2, 1)     & 100K  & 8     & 24    & 352,312   &  (11, 184) \\
Taxi$^1$           & (4, 1)     & 20K   & 50    & 150   & -         & (80, 1815) \\
SkillTeaching$^2$   & (12, 4)   & 100K  & 10    & 170   & -          & (137, 4878) \\
Elevators$^2$       & (13, 4)   & 200K  & 10    & 180   & -          & (143, 5390) \\
CrossingTraffic$^2$ & (18, 4)   & 100K  & 15    & 345   & -          & (82, 2349)
\end{tabular}
\end{center}
\end{small}
\vspace{-0.2in}
\caption{\small Superscript 1 denotes simulations of Gym domains and $^2$ denotes simulations of RDDLSim domains. $|X|,|D|$ gives the numbers of state and decision variables in the domain, $|$Columns$|$ is the total number of columns in each data record. We also report the size of the learned structures. $|$SPMN$|$ and $|$RSPMN$|$ gives the sizes of the (top, template) networks respectively. ` - ' denotes that the SPMN was not learned in 12 hrs on an Ubuntu Intel i7 64GB RAM PC.}
\label{tbl:Datasets_networksizes}
\vspace{-0.1in}
\end{table}

For each data set in the testbed, we learn an SPMN using the LearnSPMN algorithm. This was made possible by padding the sequences so that all sequences have the same length -- on reaching a terminal state, the agent stayed in that state until the length of the sequence is $T$. The top and template networks of a RSPMN were learned using our {\sc LearnRSPMN} (Algorithm~\ref{alg:LearnRSPMN}). We show the sizes of the learned networks as the total number of nodes in each, in the ultimate two columns of Table~\ref{tbl:Datasets_networksizes}. Notice the blow up in the sizes of the SPMNs learned for the sequential data sets. The SPMN has many repeated structures for the state distributions over time. For the larger RDDLSim domains, the sizes of the top and template networks also grow but we did not observe a disproportionate growth, while the SPMNs could not be learned. 

\begin{table}[!ht]
\begin{small}
\begin{center}
\begin{tabular}{l|c|c|c|c|c|c|c}
\hline
& \multicolumn{3}{c|}{\bf MEU} & \multicolumn{2}{c|}{\bf Average reward} & \\
\cline{2-4}  \cline{5-6}  
\bf Data set & \bf Optimal & \bf RSPMN & \bf SPMN & \bf RSPMN & \bf BCQ & $\Delta$ \% & \bf LL (RSPMN)\\
\hline
GridUniverse    & 6      & 6       &  6     & 5.9   & 5.9   & 0    & -0.87 \\
FrozenLake      & 0.8   & 0.818   & 0.13    & 0.8   & 0.3   & 62.5 & -6.17 \\
Maze            & 0.966  & 0.966   & 0.052  & 0.96  & 0.96  & 0    & -0.86 \\
Taxi            & 8.9    & 9       & -      & 8.9   & -200  & 60.25   & -2.45 \\
SkillTeaching   & -3.022 & -3.06   & -      &       & -5.42 & 83.3    & -2.09 \\
Elevators       & -7.33  & -7.47   & -      &       & -1.52     & 80   & -4.8 \\
CrossingTraffic & -4.428 & -4.425  & -      &       & 27.98 & 94.7    & -8.44 \\
\hline
\end{tabular}
\end{center}
\end{small}
\vspace{-0.1in}
\caption{\small Our key results comparing the MEUs of the optimal policy, learned \rspmn{}s, SPMNs, and batch-constrained Q-learning. $\Delta$ \% gives the policy deviations between the policies obtained from \rspmn{} and the optimal ones.}
\label{tbl:MEU}
\vspace{-0.1in}
\end{table}

\vspace{-0.1in}
\paragraph{MEU and policy comparisons} Table~\ref{tbl:MEU}, which reports the key results, compares the MEU from the start state of each domain as obtained by evaluating the learned \rspmn{}s and any learned SPMNs with the (near-)optimal values. We obtained the latter from converged DQNs for the Gym domains and by solving the MDP using value iteration for the RDDLSim domains. Observe that \rspmn{}s yield MEUs that are very close to the optimal and significantly better than those from the learned SPMNs. Clearly, the sequential data sets do not have sufficient episodes to learn high-quality SPMNs.

\rspmn{}s and SPMNs also represent a model of the environment as present in the data, which then plays a role in the MEU and policy computation. So, how well did the structure learning method capture the environment dynamics? To answer this question, we simulated the policies obtained from the RSPMNs in their respective Gym environments and noted down the average rewards.~\footnote{We are currently implementing the RDDLSim domains in Gym to allow simulations of our policies.} Table~\ref{tbl:MEU} shows that these average rewards nearly match the MEUs obtained directly from the RSPMNs. This implies that the networks are modeling the environment accurately. We also compare with the average rewards of policies learned by the discrete batch-constrained Q-learning~\citep{fujimoto2019benchmarking} on the data sets as a baseline, a technique similar to DQNs but constrained to learning from a batch of data. For RDDLSim domains, we report on the Q-values of the start states. Clearly, BCQ expects far more data to learn a good policy.  

Next, we report the deviation in policy learned by the RSPMN from the optimal one. This is the total number of states where the actions differ and reported as a percentage $\Delta$ \% of the total number of states. Notice the large deviations for FrozenLake, Taxi, and the RDDLSim domains although the learned policies show MEUs close to the optimal. This is likely due to the presence of multiple optimal policies in these large domains. For the sake of completeness, we also report the log likelihoods of the models in the last column.     

\begin{table}[!ht]
\begin{center}
\begin{small}
\begin{tabular}{l|c|c|c|c|c}
\hline

& \multicolumn{2}{c|}{\bf Template learning} & & \multicolumn{2}{c}{\bf MEU eval}\\
\cline{2-3} \cline{5-6}
\bf Data set & \bf Initial & \bf Final & \bf SPMN learning & \bf RSPMN & \bf SPMN\\ 
\hline
GridUniverse    & 2m 26.33s     & 1m 1.49s       & 4h 25m 31.72s & 0.72s    & 8.8s \\
FrozenLake      & 1m 49.8s      & 2m 02.78s       & 12h 5m 40.77s & 23.21s  &  1m 26.85s \\
Maze            & 2m 51.19s     & 54.84s           & 2h 31m 49.51s & 0.62s  & 24.87s \\
Taxi            & 9m 21.79s     & 2h 28m 15.75s   & - & 18.45s  & - \\
SkillTeaching   & 59m 5.87s     & 29m 28.49s       & - & 3.84s  & - \\
Elevators       &  1h 19m 3.91s & 4h 19m 29.53s   & - & 20s  & - \\
CrossingTraffic &  8m 46.14s & 1h 37m 53.17s   & - & 18.45s  & - \\
\hline
\end{tabular}
\end{small}
\end{center}
\vspace{-0.1in}
\caption{\small Initial and final template learning times are shown in first two columns. These are significantly less than those learning the large SPMNs. Run times of MEU evaluations on the learned SPMNs and \rspmn{}s are shown next.  
}
\label{tbl:times}
\vspace{-0.1in}
\end{table}

Table~\ref{tbl:times} shows our final set of results on the clock time it takes for learning the initial template structure, learning the final template of the \rspmn{} and learning the large SPMNs when possible. The time to learn an SPMN was capped at about 12 hrs. Observe that both learning and evaluating the large SPMNs takes a {\em few orders of magnitude} longer than learning the templates. However, the template learning times also increase for the larger RDDLSim domains with Elevators taking more than 4 hours. On the other hand, the MEU evaluation remains quick for all the domains taking less than a minute.

\vspace{-0.1in}
\section{Related Work}
\label{sec:related}
\vspace{-0.05in}

Melibari et al.~(\citeyear{Melibari16:Dynamic}) presented dynamic (or recurrent) SPNs that generalize SPNs to sequence data. A template network is defined that can be  repeated as many times as needed. A valid SPN is produced by capping the repeated templates with a top and a bottom network. This enables learning an SPN and performing probabilistic inference over data with varying sequence lengths.

An invariance property for the template if met yields recurrent SPNs that are valid. More recently, an online structure learning algorithm has been presented~\citep{Kalra2018} for these SPNs. \rspmn{}s can be viewed as a synergistic integration of some of the temporal concepts of recurrent SPNs with SPMNs; thereby generalizing the model to decision making. Furthermore, the handcrafted template structure in recurrent SPNs is mostly fixed (with just the number of interface nodes allowed to change) while {\sc LearnRSPMN} generates the entire template from data. 

SPNs are related to other graphical models for probabilistic inference such as arithmetic circuits~\citep{park2004differential} and AND/OR graphs~\citep{dechter2007and}. Bhattacharjya and Shachter~(\citeyear{bhattacharjya2007evaluating}) proposed decision circuits as a representation that ensures exact evaluation and solution of influence diagrams in time linear in the size of the network. A decision circuit extends an arithmetic circuit with max nodes for optimized decision making, which is analogous to how SPMNs extend SPNs. However, decision circuits are obtained by compiling IDs. Previous work has shown
that SPMNs are efficiently reducible to decision circuits in time that is linear in the size of the SPMN~\citep{Melibari16:Sum}. However, no
dynamic extension of decision circuits has been presented nor any algorithms to learn decision circuits directly from data.

In contrast to online reinforcement learning, offline (also
labeled as batch) learning seeks to derive an optimal policy from a given set of prior experiences.
This set is analogous to our simulations, and may either be fixed or allowed to grow. While offline
reinforcement learning is not as well studied as its online counterpart, the general approach is to modify
online techniques for use in batch contexts. Prominent methods, such as experience replay~\citep{Lin92:Self} and fitted Q-iteration~\citep{ernst2005tree}, are model-free and utilize the Q-update rule synchronously over all data until convergence. Recently, methods based on deep neural networks such as BCQ~\citep{fujimoto2019benchmarking} have appeared. In contrast, \rspmn{}s offer a {\em model-based approach} to learning the policy from data, which is provably tractable in the size of the network. Additionally, we established \rspmn{}'s favorable performance in comparison to BCQ. 

\vspace{-0.1in}
\section{Concluding Remarks}
\label{sec:conclusion}
\vspace{-0.05in}

We presented \rspmn{}s, a new graphical framework to model sequential decision-making problems in perfectly-observed contexts. These recurrent generalizations of SPMNs do not suffer from an exponential blow-up in size with sequence length, and can be learned directly from data using our learning algorithm. Researchers and practitioners can utilize existing data that fits the requisite schema from their rehearsals or ``dry runs'' of tasks that could benefit from automated decision making. They can also consider collecting and retaining appropriate data from computerized or real-world simulations of the tasks. \rspmn{}s are also useful for off-policy evaluations~\citep{gottesman2019combining} where the environment model is learned from the data, which can be then used for evaluating different policies.

\small
\bibliography{main}
\vspace{-0.1in}

\normalsize

\appendix
\section{Proofs}

\vspace{-0.1in}
\Validity*

\begin{proof} We sketch a proof by induction.
By assumption,
\begin{itemize}
    \item In the top network, all sum nodes are complete and product nodes are decomposable i.e. top network is valid
    \item The template network is sound, i.e., all sum nodes are complete, product nodes are decomposable, max nodes are complete and unique.
\end{itemize}

\textbf{Base case:} We prove that the RSPMN formed by interfacing a top network and a single template network is valid.
The relation between scopes of leaf latent interface nodes in the top network  with the root interface nodes $Ir$ of template network can be inferred from bijective mapping $f$ as,

\begin{align}
    scope(ir_{i}) = scope(ir_{j}) \Rightarrow scope(l_{i}^{top}) = scope(l_{j}^{top}), \label{eqn:1} \\
    (l_{i}^{top}, l_{j}^{top}) \in L, (ir_i, ir_j) \in Ir, f(l_i) \rightarrow ir_i, f(l_j) \rightarrow ir_j \nonumber
\end{align}

Since template network is sound,
\begin{align}
    scope(ir_i) = scope(ir_j), \forall (ir_i, ir_j) \in Ir
\label{eqn:2}
\end{align}

From \ref{eqn:1} and \ref{eqn:2}, the scopes of leaf latent interface nodes of top network become,
\begin{align}
    scope(l_i) = scope(l_j), \forall (l_i, l_j) \in L
\label{eqn:3}
\end{align}

This means that all the leaf latent interface nodes in top network have same scope. Under this condition and assumption, all the sum nodes of the top network are complete and product nodes are decomposable.

Next, the template network is sound. This means the scopes of all leaf latent interface nodes of template network are same because, 

\begin{align}
    & scope(ir_i) = scope(ir_j), \forall (ir_i, ir_j) \in Ir \label{eqn:4} 
\end{align}  

From bijective mapping $f(L) \rightarrow Ir$, we can infer
\begin{align}
    & scope(ir_i) = scope(ir_j) \Rightarrow scope(l_i) = scope(l_j) 
    \label{eqn:5} \\
    & (ir_i, ir_j) \in Ir, (l_i, l_j) \in L \nonumber
\end{align}  
From \ref{eqn:4} and \ref{eqn:5}, the scopes of leaf latent interface nodes of template network become,
\begin{align}
    scope(l_i) = scope(l_j), \forall (l_i, l_j) \in L
\label{eqn:6}
\end{align}

Under this condition and soundness of template, all sum nodes of template are complete, product nodes are decomposable and max nodes are complete and unique.

Now, when the top network is interfaced with a single template network, the scopes of leaf latent interface nodes of top network change based on relation with root interface nodes $Ir$ at $t=0$ as below, 

\begin{align}
    & scope(ir_{i}^{0}) = scope(ir_{j}^{0}) \Rightarrow scope(l_{i}^{top}) = scope(l_{j}^{top}), 
    \label{eqn:7} \\
    & (l_{i}^{top}, l_{j}^{top}) \in L, (ir_{i}^{t+1}, ir_{j}^{t+1}) \in Ir, f() \rightarrow ir_i, f(L_j) \rightarrow ir_j \nonumber
\end{align}

From \ref{eqn:4} we have,
\begin{align}
    scope(ir_{i}^{0}) = scope(ir_{j}^{0}), \forall (ir_{i}^{0}, ir_{j}^{0}) \in Ir 
\label{eqn:8}
\end{align}

From \ref{eqn:7} and \ref{eqn:8},
\begin{align}
    scope(l_{i}^{top}) = scope(l_{j}^{top}), \forall (l_{i}^{top}, l_{j}^{top}) \in L
\label{eqn:9}
\end{align}

The condition from \ref{eqn:9} is equivalent to the condition from \ref{eqn:3}. This means the scopes of leaf latent interface nodes of top network have not changed after interfacing with the template network. So, all the sum nodes of top network are complete and product nodes are decomposable even after interfacing with template network. Since no scope is changed in template network after interfacing, all sum nodes are complete, product nodes are decomposable and max nodes are complete and unique in the template network. Therefore the SPMN formed after interfacing top network with single template network is valid. 

\textbf{Induction hypothesis:} Let us assume that the SPMN formed after interfacing a top network and the template repeated $t$ times is valid, i.e., all sum nodes are complete, product nodes are decomposable and max nodes are complete and unique. Let this SPMN be $R$

\textbf{Inductive step:} We now prove that an SPMN formed by interfacing one more template network (template network repeated $(t+1)$ times in total) with $R$ is a valid SPMN. 

Since the template network is sound, as we have shown in \ref{eqn:4}, \ref{eqn:5} and \ref{eqn:6}, we can show that for template at $t+1$,
\begin{align}
    & scope(ir_{i}^{t+1}) = scope(ir_{j}^{t+1}), \forall (ir_{i}^{t+1}, ir_{j}^{t+1}) \in Ir  \label{eqn:10} \\
    & scope(l_{i}^{t+1}) = scope(l_{j}^{t+1}), \forall (l_{i}^{t+1}, l_{j}^{t+1}) \in L  \label{eqn:11} 
\end{align}
and for template at $t$,
\begin{align}
    & scope(ir_{i}^{t}) = scope(ir_{j}^{t}), \forall (ir_{i}^{t}, ir_{j}^{t}) \in Ir  \label{eqn:12} \\
    & scope(l_{i}^{t}) = scope(l_{j}^{t}), \forall (l_{i}^{t}, l_{j}^{t}) \in L  \label{eqn:13} 
\end{align}

When the template at $t$ is interfaced with the template at $t+1$, the scopes of leaf latent interface nodes of template at $t$ relate to root interface nodes of template at $t+1$ as follows,
\begin{align}
    scope(ir_{i}^{t+1}) = scope(ir_{j}^{t+1}) \Rightarrow scope(l_{i}^t) = scope(l_{j}^t), \label{eqn:14} \\
    (l_{i}^t, l_{j}^t) \in L, (ir_{i}^{t+1}, ir_{j}^{t+1}) \in Ir, f() \rightarrow ir_i, f(L_j) \rightarrow ir_j \nonumber
\end{align}

From \ref{eqn:10} and \ref{eqn:14} we have,
\begin{align}
    scope(l_{i}^{t}) = scope(l_{j}^{t}), \forall (l_{i}^{t}, l_{j}^{t}) \in L  \label{eqn:15} 
\end{align}

The condition from \ref{eqn:15} is equivalent to the condition from \ref{eqn:13}. This means the scopes of leaf latent interface nodes of template network at $t$ have not changed after interfacing with the template network at $t+1$. From inductive hypothesis, SPMN $R$ is valid. Since there is no change in scopes of any of the nodes in $R$ after interfacing with template at $t+1$, all sum nodes are complete, product nodes are decomposable and max nodes are complete and unique in $R$. Since no scope is changed in template network at $t+1$ after interfacing, all sum nodes are complete, product nodes are decomposable and max nodes are complete and unique in the template network at $t+1$. So, all sum nodes are complete, product nodes are decomposable and max nodes are complete and unique in the SPMN formed after interfacing $R$ with template network at $t+1$. Therefore, the SPMN formed after interfacing $R$ with one more template network (template repeated $t+1$ times in total) is valid.

\end{proof}
\vspace{-0.1in}
\section{Algorithms}

Algorithm~\ref{alg:LearnRSPMN} presents the four main steps involved in learning the \rspmn{} from data which we refer to as {\sc LearnRSPMN}

\vspace{-0.1in}
\paragraph{[Step 1] Learn an SPMN from 2-time step data} The first step of {\sc LearnRSPMN} is to use LearnSPMN algorithm
to learn a valid SPMN, $\Scal{}^{t=2}$, from 2-time step data by wrapping the data set with $P^\prec$ as the partial order. This process is shown in algorithm

\begin{algorithm}[!ht]
\DontPrintSemicolon
    
    \SetKwInOut{Input}{input}
    \Input
    {
        Dataset: $\langle \tau^0, \tau^1, \ldots, \tau^T \rangle_e$ where $e = 1,2, \ldots, E$,
        Partial Order: $P^\prec$
    }
    
    \SetKwInOut{Output}{output}
    \Output
    {
        SPMN from 2-time step data: $\Scal{}^{t=2}$\
    }

    $w^0$ $\leftarrow$ Empty,
    $w^1$ $\leftarrow$ Empty
    
    \For{$e$ in $1,2, \ldots, E$}
    {
        \For{$t$ in $ 0, 1, \dots T-1$}
        {
            $w^0$ $\leftarrow$ $w^0.add(\tau^{t})$ \;
            $w^1$ $\leftarrow$ $w^1.add(\tau^{t+1})$

        } 
    }
        
    $W$ $\leftarrow$ $\langle w^0, w^1 \rangle$ \\
    $\Scal{}^{t=1}$ $\leftarrow$ {\sc LearnSPMN}$(W$, $P^\prec)$
    
\caption{SPMN from 2-time step data}
\vspace{-0.05in}
\label{alg:Two step network}
\end{algorithm}

\paragraph{[Step 2] Obtain top network and $Ir$ nodes} First we extract a 1-time step network $\Scal{}^{t=1}$ from $\Scal{}^{t=2}$ and we obtain the nodes in $Ir$ using $\Scal{}^{t=1}$. This is shown in Algorithms~\ref{alg:top network and set of  root interface nodes} and~\ref{alg:RecurInterfaceChidlren}

\begin{algorithm}[!ht]
\DontPrintSemicolon
    
    \SetKwInOut{Input}{input}
    \Input
    {
        Two step SPMN: $\Scal{}^{t=2}$
    }
    
    \SetKwInOut{Output}{output}
    \Output
    {
        Top Network $\Scal{}_O$, Set of  root interface nodes $Ir$
    }

    \SetKwFunction{IrC}{$IrChildren$}
    
    $Queue \leftarrow \Scal{}^{t=2}.root$ \;
    \While{$Queue$ is not $\emptyset$}
    {
        $node  \leftarrow Queue.dequeue $ \;
        \If{$node$ is product}
        {
            \For{each child $c$ in the set of node's children $C_n$}
            {
                \If{$c$ does not have any variable of $(I_0, d_1, I_1, d_2, \ldots, I_{m-1}, d_m, I_m, u)^1$ in its scope}
                {
                Remove $c$ from $C_n$ \;
                }
            }
        }
    }
One step network $\Scal{}^{t=1} \leftarrow$  remaining $\Scal{}^{t=2}$ \
  
    set of nodes seen $E \leftarrow \emptyset$ ;  root interface nodes $Ir \leftarrow \emptyset$ \;
    $Queue \leftarrow  \Scal{}^{t=1}.root$ \;
    
    \While{$Queue$ is not $\emptyset$}
    {
        $node  \leftarrow Queue.dequeue $ \;
        \If{($node$ is product and all $(I_0, d_1, I_1, d_2, \ldots, I_{m-1}, d_m, I_m, u)^0$ are in $node.scope$)}
        {
            topProdChildren $C \leftarrow \emptyset$ \;
            \tcc{ $c.scope$ must be a proper subset}
            \If{(each $c.scope \subset (I_0, d_1, I_1, d_2, \ldots, I_{m-1}, d_m, I_m, u)^1,  \forall c \in node.children $)}
            {
                $C \leftarrow node.children$ \; 
                $node.children \leftarrow \emptyset $
            }
            
            \If {$C$ is not $\emptyset$}
            {
                Create Product node $P$ ; $P.children \leftarrow C$ \;
                $R \leftarrow $ \IrC{$P$} \; 
                latent interface nodes $L \leftarrow \emptyset$ \;
                \For{each $ir children$ set $ir_C$ in $R$}
                {
                    Create interface root product node $ir$ \; 
                    $ir.children \leftarrow ir_C$ \;
                    Create latent interface node $l$ with a bijective mapping $f(l) \rightarrow ir$ \;
                    $L \leftarrow L \cup l$ ; $Ir \leftarrow Ir \cup ir$ \;
                    
                }
                Create Sum node $S_L$ with $S_L.children \leftarrow L$ \;
                $node.children \leftarrow node.children \cup S_L$ 
            }
            \Else
            {
                \For{each child $c$ in the set of node's children $C_n$}
                {
                    \If{$c$ is not in $E$}
                    {
                        Add $c$ to $E$ \; Enqueue $c$ into $Queue$
                    } 
                }
            }
        }
    }
Top Network $\Scal{}_O \leftarrow $ remaining $\Scal{}^{t=1}$ \; 
Set of  root interface nodes $Ir$
\caption{Create top network and set of  root interface nodes}
\label{alg:top network and set of  root interface nodes}
\end{algorithm}

\begin{algorithm}[!ht]
\DontPrintSemicolon

    \SetKwInOut{Input}{input}
    \Input
    {
        Product node $P$
    }
    \SetKwInOut{Output}{output}
    \Output
    {
        set of interface root node children sets
    }
  
    \SetKwFunction{IrC}{$IrChildren$}
    
    \SetKwProg{Fn}{Function}{:}{}
                  \Fn{\IrC{$P$}}{
                       \If{(any $c$ is $ Sum, \forall c \in P.children$)}
                       {
                            \For{each $c$ in $P.children$}
                            {
                                stateVars $ X \leftarrow \emptyset$ \;
                                \If{$c$ is Sum}
                                {
                                    \tcc{a set of sets}
                                    set of $IrChildren$ sets $R \leftarrow \emptyset $ \;
                                    \For{each $c_p$ in $c.children$}
                                    {
                                        $irChildren$ set $ir_C \leftarrow$ \IrC{$c_p$}  \;
                                        $R \leftarrow R \cup ir_C$
                                    }
                                }
                                \Else
                                {
                                    $X \leftarrow X \cup c$
                                }
                             }
                            \For{each $ir_C$ in $R$}
                            {
                                $ir_C \leftarrow ir_C \cup X$
                            }
                            \KwRet $R$\;
                       }
                    \Else
                    {
                      \KwRet $\{\{P\}\}$\;  
                    }
                        
                  }
           
\caption{make sets of $Ir$ children}
\vspace{-0.05in}
\label{alg:RecurInterfaceChidlren}    
\end{algorithm}

\paragraph{[Step 3] Building an initial template}
$Ir$ from the previous step containing the root interface nodes forms a basis for creating the initial template. The sum node $S_L$ whose children are the set of leaf latent nodes $L$ is added to $Ir$ to obtain initial template structure as shown in Algorithm~\ref{alg:InitialTemplateStructure}

\begin{algorithm}[!ht]
\DontPrintSemicolon
    
    \SetKwInOut{Input}{input}
    \Input
    {
        Set of root interface nodes $Ir$
    }
    
    \SetKwInOut{Output}{output}
    \Output{Initial template root interface nodes $Ir$}
    
    $L \leftarrow \emptyset$ \;
    \For {each $ir$ in $Ir$}
    {
        Create a Latent interface node $l$ and $f(l) \rightarrow ir$ \;
        $L \leftarrow L \cup l$
    }
    
    \For {each $ir$ in $I$}
    {
        set of nodes seen $E \leftarrow \emptyset$ \;
        $Queue \leftarrow ir.root$ \;
        \While{$Queue$ is not $\emptyset$}
        {
            $node  \leftarrow Queue.dequeue $ \;
            \If{ (each $c$ is Leaf node, $\forall c \in node.children$)}
            {
                
                Create Sum node $S_L$ \tcc{bottom sum interface node} 
                $S_L.children \leftarrow L$ \;
                $S_L.weights \leftarrow 1/numOf(S_L.children)$ \;
                
                \If{$node$ is product}
                {
                    Add $S_L$ as child of product node
                }
                \Else
                {
                    \For {each $l$ that is Leaf node in set of $node.children$ $C_n$}
                        {
                        Create Product node $P$ \;
                        $P.children \leftarrow S_L \cup l$ \;
                        $C_n \leftarrow C_n \setminus l$ \;
                        $C_n \leftarrow C_n \cup P$
                        }
                    
                }
            }
            \Else
            {
               \For{each $c \in  node.children$}
                    {
                        \If{$c$ is not in $E$}
                        {
                            Add $c$ to $E$ \;
                            Enqueue $c$ into $Queue$
                        }
                    }
            }
        }
    }
    Initial template root interface nodes $Ir$ $\leftarrow$ set of root interface nodes $Ir$

\caption{Create initial template network}
\label{alg:InitialTemplateStructure}
\end{algorithm}

\paragraph{[Step 4] Learning the final template network} By updating the edge weights of the outgoing edges from the sum node $S_L$ whose children are set of leaf latent nodes $L$, we can learn the final template structure as shown in Algorithm~\ref{alg:FinalTemplateStructure}.
\begin{algorithm}[!ht]
\DontPrintSemicolon
    
    \SetKwInOut{Input}{input}
    \Input
    {
        Initial template root interface nodes $Ir$; Top Network $\Scal{}_O$
    }
    
    \SetKwInOut{Output}{output}
    \Output{Final template root interface nodes}
    
    $T \leftarrow$ length of sequence \;
    $ll^{-1}, ll^0, \dots ,ll^{T-1} \leftarrow$ EmptyList \;
    \For {each time step $t$ from $T-1 \dots 0$}
    {
        \If{ $t = T$ }
        {   
            Remove leaf latent nodes from $Ir$ \;
            Assign leaf values $(I_0, d_1, I_1, d_2, \ldots, I_{m-1}, d_m, I_m, u)^t$ to leaves in $Ir$ \;
            $ll^{t} \leftarrow$ bottomUpPass($Ir$) \;
        } 
        \Else
        {
            Assign leaf values $(I_0, d_1, I_1, d_2, \ldots, I_{m-1}, d_m, I_m, u)^t$ to leaves in $Ir$ \;
            Assign $Ir$ values from $ll^{t+1}$ to latent leaves in $Ir$ \;
            $ll^{t} \leftarrow$ bottomUpPass($Ir$) \;
        }
    }
    Assign $Ir$ values from $ll^{0}$ to latent leaves of $\Scal{}_O$ \;
    $ll^{-1} \leftarrow$ bottomUpPass($\Scal{}_O$) \;
    
    \For {each time Step $t$ from $-1 \dots m-1$}
    {
        \If{ $t = -1$ }
        {
            TopDownPass($\Scal{}_O$) \;
            $l \leftarrow$ leaf latent interface node reached
        } 
        \Else
        {
            $I_{l} \leftarrow $ root interface node corresponding to $l$ \;
            TopDownPass($I_{l}$) \;
            Increment counts on outgoing edges visited on sum nodes \;
            $l \leftarrow$ leaf latent interface node reached
        } 
        
    Update weights on sum node $S_L$ whose children are $L$ in $Ir$ using counts \;
    Drop branches with zero counts on $S_L$ \;
    Final template root interface nodes $Ir \leftarrow Ir$
    }
\caption{Final template network}
\label{alg:FinalTemplateStructure}

\end{algorithm}

\vskip 0.2in

\end{document}